\documentclass[11pt,draftcls,onecolumn,journal]{IEEEtran}
\usepackage[doublespacing]{setspace}

\usepackage{natbib}
\usepackage{amsmath,amsthm,amsfonts,amssymb,mathtools}
\usepackage{algorithmic}
\usepackage{algorithm}
\usepackage{array}
\usepackage[caption=false,font=normalsize,labelfont=sf,textfont=sf]{subfig}
\usepackage{textcomp}
\usepackage{stfloats}
\usepackage{url}
\usepackage{verbatim}
\usepackage{graphicx}
\usepackage{hyperref}
\usepackage{booktabs}
\usepackage{xcolor}

\newcommand{\rev}[1]{{\color{black}#1}}

\newtheorem{theorem}{Theorem}[section]

\newtheorem{lemma}[theorem]{Lemma}

\newtheorem{conjecture*}{Conjecture}

\newcommand{\calP}{\mathcal{P}}

\newcommand{\calN}{\mathcal{N}}

\newcommand{\E}{\mathbb{E}}
\newcommand{\R}{\mathbb{R}}
\newcommand{\W}[0]{\mathcal{W}_2}

\usepackage[doublespacing]{setspace}

\usepackage{amsmath,amssymb}

\begin{document}

\title{Flow-Based Generative Models as  Iterative Algorithms in Probability Space}

\author{ Yao Xie,~\IEEEmembership{Senior~Member,~IEEE}, Xiuyuan Cheng
\thanks{Yao Xie is with the H. Milton Stewart School of Industrial and Systems Engineering (ISyE) at the Georgia Institute of Technology. Xiuyuan Cheng is with the Mathematics Department at Duke University.}
}

\markboth{IEEE Signal Processing Magazine,~Vol.~XX, No.~XX, July~2024}%
{\MakeLowercase{\textit{et al.}}: Author Guidelines for Special Issue Articles of IEEE SPM}

\maketitle



\vspace{-0.5in}
\section{Introduction}

Generative AI (GenAI) has revolutionized data-driven modeling by enabling the synthesis of high-dimensional data in fields such as image generation, large language models (LLMs), biomedical signal processing, and anomaly detection. Among GenAI approaches, diffusion-based (see, e.g., \cite{song2021score,ho2020denoising,song2021scorebased}) and flow-based \cite{nflow_review,dinh2014nice,dinh2016density,kingma2018glow,grathwohl2018ffjord,chen2018neural,lipman2023flow,albergo2023building,liu2022rectified,xu2022jko} generative models have gained prominence due to their ability to model complex distributions sample generation and  density estimation.


Flow-based models leverage invertible mappings governed by Ordinary Differential Equations (ODEs), unlike diffusion models, which rely on iterative denoising through Stochastic Differential Equations (SDEs). The design of flow-based models provides a direct and deterministic transformation between probability distributions, enabling exact likelihood estimation and fast sampling. This makes them particularly well-suited for tasks requiring density estimation and likelihood evaluation, such as anomaly detection and probabilistic inference.
While empirical advancements in generative models have been substantial, a deeper understanding of the design and mathematical foundations of flow-based generative models will enable both theoreticians and practitioners to broaden their adoption for diverse applications in signal processing and leverage them as a general representation of high-dimensional distributions. 


This tutorial presents an intuitive mathematical framework for flow-based generative models, viewing them as neural network-based representations of continuous probability densities. These models can be cast as particle-based iterative algorithms in probability space using the Wasserstein metric, providing both theoretical guarantees and computational efficiency. Based on this framework, we establish the convergence of the iterative algorithm and show the generative guarantee, ensuring that under suitable conditions, the learned density approximates the true distribution. By systematically building from fundamental concepts to state-of-the-art research, our goal is to guide the audience from a basic understanding to the research frontiers of generative modeling, demonstrating its impact on signal processing, machine learning, and beyond.

\begin{figure}[t]
\vspace{-0.1in}
\centering
\includegraphics[height=.35\linewidth]{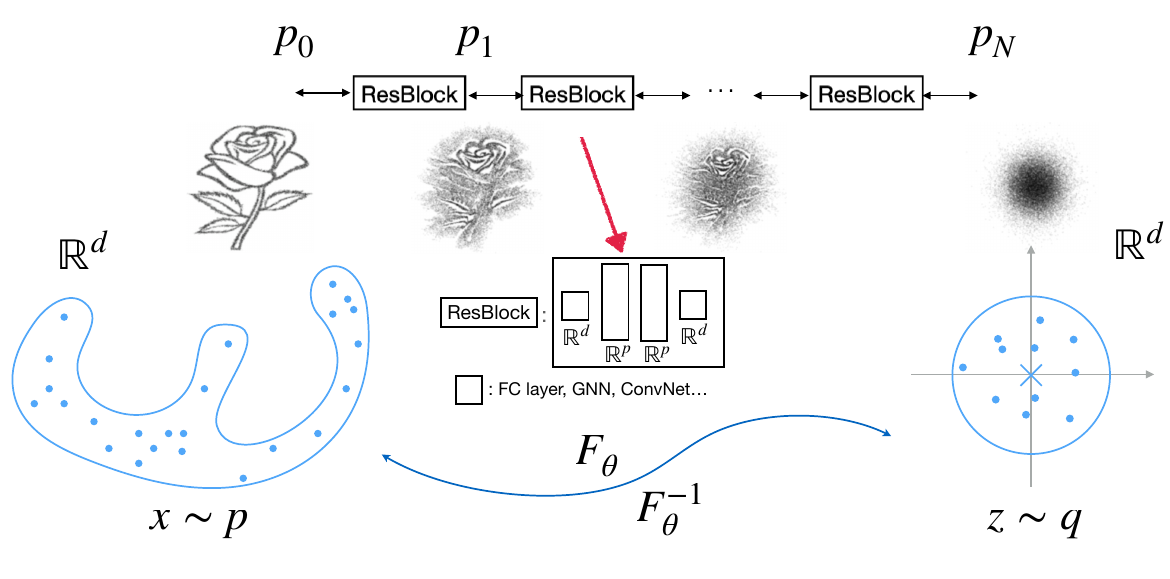}  
\vspace{-0.1in}
\caption{General setup of a flow-based generative model, where the forward process is captured through a forward mapping \( F_\theta \), and the reverse process is captured by the inverse mapping \( F_\theta^{-1} \). Arrows indicate the forward-time flow from the data distribution \( p \) to the target distribution \( q \) (typically Gaussian noise). The forward process maps \( p \) to the noise distribution \( q \), while the reverse process reconstructs \( p \) from \( q \). Both processes involve a sequence of transported densities at discrete time steps, with ResNet blocks serving as iterative steps that push densities in probability space under the Wasserstein-2 metric. 
}
\label{fig:pn}
\end{figure}

\begin{figure}[t]
\centering
\includegraphics[width=.8\linewidth]{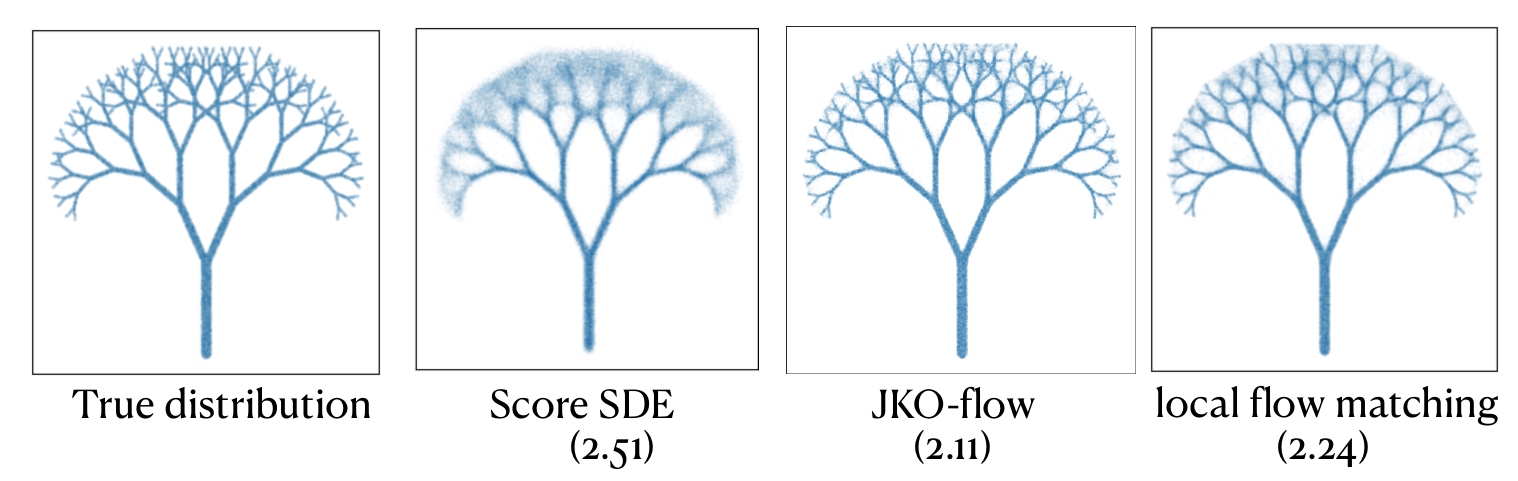}  
\vspace{-10pt}
\caption{Illustrative example of generating a two-dimensional ``fractal tree" distribution using various flow-based generative models.
The numbers in brackets represent the negative log-likelihood (NLL), where lower values indicate better performance. Note that the flow-based generative model captures finer details of the distribution more effectively, while JKO-flow \cite{xu2022jko} achieves a lower (better) NLL score, at the cost of higher computational cost than local Flow Matching \cite{xu2024local}.
}
\label{local-flow}
\vspace{-0.2in}
\end{figure}

\vspace{-0.1in}
\section{Mathematical background}

In this section, we introduce the essential mathematical background for understanding flow-based generative models. 
We start with the concept of an ODE and the {\it velocity field}, which describe the continuous transformations applied to data samples (particles) as they progress through the model.
 Following this, we examine how data density evolves using the {\it continuity equation}, which characterizes changes in probability density over time within the transformation. 
 We then introduce the SDE and the associated  Fokker-Planck equation, 
 and explain its relationship to the continuity equation. Finally, we provide the preliminaries of Wasserstein space and the Optimal Transport map.
 

\vspace{-0.1in}
\subsection{ODE and Continuity equation}

\rev{
An ordinary differential equations (ODE) describes the evolution of a system over time through deterministic rules, typically written as 
\begin{equation}\label{eq:flownet}
\dot{x}(t)= v(x(t), t), \quad t \in [0,T],
\end{equation}
where the data point (particle) $x(t)\in \mathbb R^d$, and $\dot x(t)= d x(t) / dt$.
The given mapping $v(x,t):\mathbb R^d \rightarrow \mathbb R^d$ is call the {\it velocity field}, 
which describes how the a particle evolves, and the velocity field can vary over the position $x$ and time $t$. 
ODEs are foundational tools in modeling dynamical systems and appear widely across applied mathematics and physics, and more background can be found in textbooks such as \cite{Sideris2013OrdinaryDE}.}
For the velocity field,  we also write $v(\cdot, t) = v_t(\cdot )$.

As the participle evolves according to the dynamic, its underlying distribution also evolves over time.
For the continuous-time  ODE dynamic \eqref{eq:flownet}, 
let $P$ be the data distribution with density $p$, $x(0) \sim p$,
and denote by $\rho_t (x) = \rho (x,t)$ the probability density of $x(t)$.
The evolution of $\rho_t$ is governed by the {\it continuity equation} (CE) as 
\begin{equation}\label{eq:liouville}
\partial_t \rho_t + \nabla \cdot (\rho_t  v_t) = 0,
\end{equation}
from $\rho_0 = p$. 
Here the divergence operator follows the standard definition in vector calculus: for a vector field $u(x)=[ u_1(x), \ldots, u_d(x)]$ for $x\in\mathbb R^d$, $\nabla\cdot u(x) =\sum_{j=1}^d \partial u_j(x)/\partial x_j$.
\rev{
The CE, which expresses the conservation of mass (or probability) over time, is a fundamental principle in physics and fluid dynamics.
In statistical physics, this formalism underlies the macroscopic description of particle systems, where probability flows like a conserved fluid.
In Section \ref{subsec:sde-fpe-review},
we will compare CE of an ODE to the Fokker-Planck equation (FPE) of an SDE, 
as both equations describe density evolution.
This connection will also bridge flow and diffusion based generative models.}

Mathematically, the solution trajectory of ODE is well-defined: given the initial value problem, ODE is well-posed under certain regularity conditions of the velocity field $v$, 
meaning that the ODE solution exists, is unique, and continuously depends on the initial value. Informally, the CE \eqref{eq:liouville} provides insights of how data distribution changes from a simple initial state into a more complex target distribution: 
If the algorithm can find a  $v_t$ such that $\rho_T$ at some time $T$ is close to $q$, then one would expect the reverse-time flow from $t=T$ to $t=0$ to transport from $q$ to a distribution close to $p$. Normalizing flow models drive data distribution towards a target distribution $q$ typically normal, $q = \calN(0, I_d)$, per the name ``normalizing.'' 

\vspace{-0.1in}
\subsection{SDE and Fokker-Plank equation}\label{subsec:sde-fpe-review}

\rev{Another popular type of generative model, namely the Diffusion Models, is based on stochastic differential equation (SDE), 
for example, the Ornstein–Uhlenbeck (OU) process.
The OU process  is a classic example of a SDE modeling noisy dynamics with mean-reverting behavior. 
The OU process in  in $\R^d$ takes the form 
\begin{equation}\label{eq:OU-SDE}
    dX_t = - X_t dt + \sqrt{2} dW_t,
\end{equation}
where $dW_t$ is standard Brownian motion.
Introductory treatments of SDE can be found in \cite{evans2012introduction,oksendal2013stochastic}.
The SDE used in Diffusion models, see e.g. \cite{song2021scorebased}, is based on the OU process \eqref{eq:OU-SDE} under a change of time reparametrization.}
More generally, one can consider a diffusion process 
 \begin{equation}\label{eq:diffusion-sde-2}
 dX_{t}= - \nabla V(X_{t} )\,dt+ \sqrt{2}\,dW_{t}, \quad X_0 \sim P,
 \end{equation}
where $\nabla V$ denotes the gradient of a scalar function $V: \mathbb R^d \rightarrow \mathbb R$.
The OU process \eqref{eq:OU-SDE} is a special case with $V(x) = \| x\|^2/2$.

We denote by $\rho_t$  the marginal distribution of $X_t$ for $ t >0$,
and the time evolution of $\rho_t$ is described by the Fokker-Planck equation (FPE).
\rev{
The FPE describes the time evolution of probability densities under stochastic dynamics, such as Brownian motion and particle diffusion, and has deep roots in statistical physics. From this perspective, generative modeling via flow-based or diffusion models can be seen as learning or simulating physical processes that transport probability mass over time. The learned dynamics can approximate thermodynamically motivated flows, such as those that minimize free energy or entropy production. This connection to physics provides both a theoretical foundation and intuitive guidance for designing and interpreting generative trajectories, and in this tutorial we elaborate on the optimization perspective.}

For the process \eqref{eq:diffusion-sde-2}, the FPE is  written as
\begin{equation}\label{eq:FPE}
\partial_{t}\rho_t = \nabla\cdot(\rho_t \nabla V + \nabla \rho_t).
\end{equation}
Note that the density evolution through the CE \eqref{eq:liouville} and the FPE \eqref{eq:FPE} are mathematically equivalent when we set
\begin{equation}
v(x, t) = -\nabla V(x) - \nabla \log \rho(x,t).
\label{velocity_score}
\end{equation}
However, these two approaches lead to very different trajectories (as illustrated in Fig. \ref{fig:SDE_ODE}) and algorithms: the flow-based model learns the velocity field $v$, which is implicitly related to the score function $\nabla \log \rho$ that the diffusion generative model learns in the forward process and uses in the reverse-time generative process, illustrated in Fig. \ref{fig:pn}.
\rev{While closely related, flow- and diffusion-based models each offer unique strengths:
the deterministic ODE trajectory in flow model generally facilitates faster model evaluation and generation,
while Diffusion Models, which adopts SDE sampling in theory, enjoys better known theoretical convergence guarantees thanks to the SDE theories. 
It has become increasingly common to distill deterministic flow-based generators from trained Diffusion Models,
most notably with Consistency Models \cite{song2023consistency}, 
and 
Progressive Distillation \cite{salimans2022progressive},
among many recent developments.
}

 \begin{figure}[t]
\vspace{-0.1in}
\centering
\includegraphics[height=0.3\linewidth]{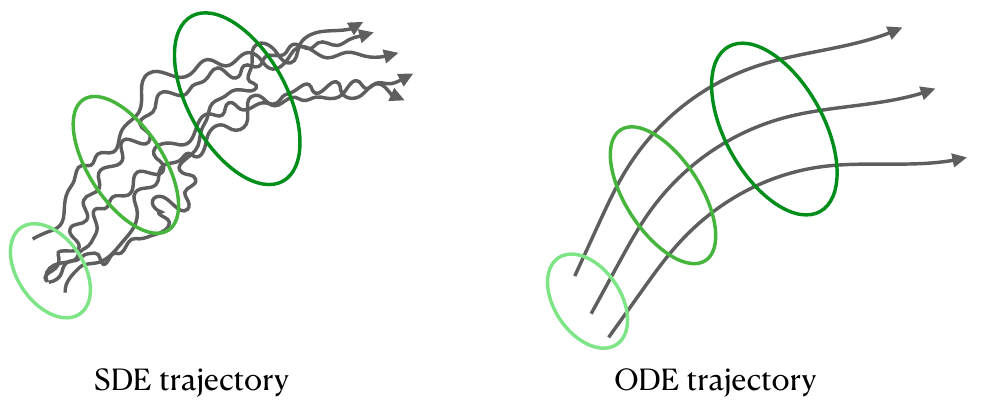}
 \vspace{-10pt}
 \caption{Trajectory of SDE versus ODE: The SDE trajectory corresponds to a diffusion model, while for ODE, the dynamics are deterministic, but the initial position of the trajectory follows a distribution.}
 \label{fig:SDE_ODE}
 \vspace{-0.1in}
 \end{figure}

\vspace{-0.1in}
 \subsection{Wasserstein space and Optimal Transport}
 
We also review the definitions of the Wasserstein-2 distance and optimal transport (OT) map, which are connected by the Brenier Theorem (see, e.g., \cite[Section 6.2.3]{ambrosio2005gradient}).
Denote by $\calP_2$ the space of probability distributions on $\R^2$ with finite second moments,
namely $\calP_2 = \{ P \text{ on $\R^d$}, \, s.t., \int_{\R^d} \| x \|^2 dP(x) < \infty \}$, and denote by $\calP_2^r$ the distributions in $\calP_2$ that have densities.
Given two distributions $\mu, \nu \in \calP_2$, the Wasserstein-2 distance $\W( \mu, \nu)$ is defined as
\begin{equation}\label{eq:ot}
    \W^2( \mu, \nu ) := \inf_{\pi \in \Pi ( \mu, \nu)}
        \int_{ \R^d \times \R^d} \|x-y\|^2 d\pi(x,y),
\end{equation}
where $\Pi (\mu, \nu)$ denotes the family of all joint distributions with $\mu$ and $\nu$ as marginal distributions.
When $P$ and $Q$ are in $\calP_2^r$ and  have densities $p$ and $q$ respectively, we also denote $\W(P,Q)$ as $\W(p,q)$. When at least $\mu$ has density, we have the following result by the Brenier Theorem, which allows us to define the optimal transport (OT) map: 
The unique minimizer of \eqref{eq:ot} is achieved at $\pi = ( {\rm I}_d, T_\mu^\nu)_\#\mu$, where ${\rm I}_d$ denotes the identity map, $T_\mu^\nu$ is the OT map  from $\mu$ to $\nu$ which is $\mu$-a.e. defined.
Here, the {\it pushforward} of a distribution $P$, by a map $F:\R^d \to \R^d$, is denoted as $F_\# P$, 
such that 
$F_\# P(A) = P(F^{-1}(A))$ for any measurable set $A$ in $\R^d$.
The minimum of \eqref{eq:ot}  also equals that of the Monge problem, namely
\begin{equation}\label{eq:ot-monge}
\W^2( \mu, \nu) = \inf_{F: \R^d \to \R^d, \, F_\# \mu = \nu }
    \int \| x -F(x) \|^2 d\mu(x).
\end{equation}








\vspace{-0.1in}
\section{Algorithm basics of generative flow models}\label{sec:algo-basics}

Normalizing Flow (NF) is a class of deep generative models for efficient sampling and density estimation. 
Compared to diffusion models, NF models \cite{nflow_review} appear earlier in the generative model literature. Fig. \ref{fig:pn} illustrates the general setup of flow-based generative models.
%
%
To be more specific, NFs are transformations that map an easy-to-sample initial distribution, such as a Gaussian, to a more complex target distribution.
Generally, an NF model provides an invertible {\it flow mapping} $F_\theta: \R^d \to \R^d$, parametrized by $\theta$ (usually a neural network), such that it maps from the ``code'' or ``noise'' $z$ (typically Gaussian distributed) to the data sample $x \sim p$, where $p$ is the unknown underlying data distributions. 
We are only given the samples from the data distribution $p$ as training data.
Once a flow model $F_\theta$ is trained, one can generate samples $x$ by computing $x= F_\theta(z)$ and drawing $z \sim q$, where $q$ is a distribution convenient to sampling in high dimensional, typically $q = \calN(0,I)$.

\vspace{-0.1in}
\subsection{Training objective and particle-based implementation}

NFs aim to learn the transform from a simple distribution to a complex data distribution by maximizing the log-likelihood of the observed data. 
In terms of algorithm, the training aims to find the flow map $F_\theta$ that minimizes the training objective.
%
%
Given a dataset \( \{x^i\}_{i=1}^m \) (referred to as ``particles''), we assume that each \( x^i \) is sampled from an unknown distribution \( p \), and want to approximate it with our flow model \( p_\theta(x) \). The model transforms a simple target distribution \( q  \), e.g., $\calN(0,I)$, into \( p_\theta \) using the invertible mapping \( F_\theta \). 
The log-likelihood of a data point \( x \) under the flow model $F_\theta$ is then
\begin{equation}\label{eq:log-ptheta-NF}
\log p_\theta(x) = \log p_z(F_\theta^{-1}(x)) + \log \left| \det J_{F_\theta^{-1}}(x) \right|,
\end{equation}
where \( p_z(z) \) is the density of the noise (Gaussian distribution), \(F_\theta^{-1} \) maps the data \( x \) back to the noise space, 
and \( \det J_{F_\theta^{-1}}(x) \) is the Jacobian determinant of the inverse transformation, capturing how the transformation scales probability mass.
To train the model, maximizing the total log-likelihood over all training samples leads to minimizing the negative log-likelihood (NLL):
\begin{equation}
\mathcal{L}(\theta) = -\frac{1}{m} \sum_{i=1}^{m} \left[ \log p_z(F_\theta^{-1}(x^i)) + \log \left| \det J_{F_\theta^{-1}}(x^i) \right| \right], \label{NLL_def}
\end{equation}
which serves as a training objective. 

Since NFs are differentiable, $\mathcal{L}(\theta)$ can be optimized using stochastic gradient descent (SGD) or variants such as Adam. 
The gradient of the loss function can be computed using backpropagation through the invertible transformations. However, depending on the way of parametrizing $F_\theta$, the computation of likelihood and backpropagation training can be expensive and challenging to scale to high dimensional data. 
The key to design a flow model is to construct $F_\theta$ for efficient training and generation, which we detail in below.


 \begin{figure}[t]
\vspace{-0.1in}
\centering
\includegraphics[height=0.3\linewidth]{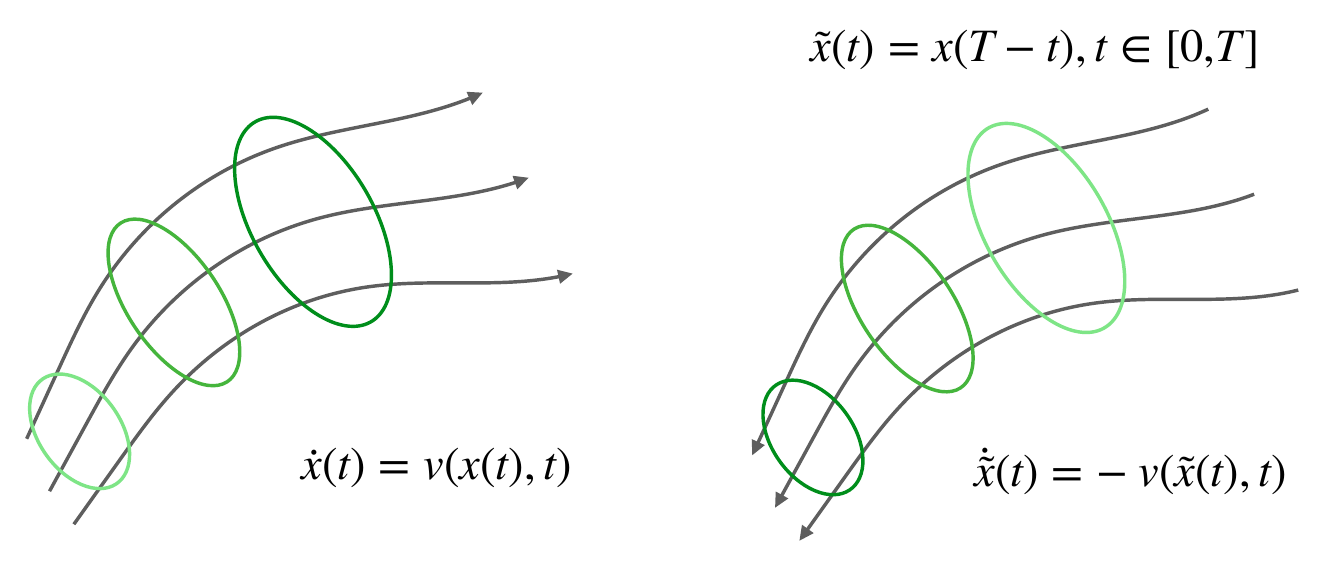}
 \caption{Reverse time dynamic of ODE: ODE can go forward or backward in time deterministically.
 }
\label{fig:reverse_time_ODE}
\vspace{-0.2in}
 \end{figure}

\vspace{-0.1in}
\subsection{Discrete-time Normalizing Flow}
Largely speaking, NFs fall into two categories: discrete-time and continuous-time.
The discrete-time NF models typically follow the structure of a Residual Network (ResNet) \cite{he2016deep} and consist of a sequence of mappings:
\begin{equation}\label{eq:resnet}
x_{n} = x_{n-1} +  f_n (x_{n-1}), \quad n = 1, \ldots, N,
\end{equation}
where each $f_n$ is a neural network mapping parameterized by the $n$-th ``Residual Block'', and $x_n$ is the output of the  $n$-th block. 
The composition of the $N$ mappings $x_{n-1} \mapsto x_n$, $n=1,\ldots, N$, together provides the (inverse of the) flow mapping $F_\theta$, 
the invertibility of which needs to be ensured by additional techniques 
and usually by enforcing each block mapping $x_{n-1} \mapsto x_n$ to be invertible. 
The computation of the inverse mapping, however, may not be direct. 
After training, the generation of $x$ is by sampling $z\sim q$ and computing $z \mapsto x$ is via the flow map computed through $N$ blocks.

Meanwhile, for the discrete-time flow \eqref{eq:resnet}, the computation of the likelihood \eqref{eq:log-ptheta-NF}
calls for the computation of the log-determinant of the Jacobian of $ f_n$.
This can be challenging for a general multivariate mapping $f_n$.
To facilitate these computations, earlier NFs such as  
NICE \cite{dinh2014nice}, Real NVP \cite{dinh2016density}, and Glow \cite{kingma2018glow}
adopted special designs of the neural network layer type in $f_n$ in each block. 
While the special designs improve the computational efficiency of invertibility and log determinant of Jacobian, it usually restrict the expressiveness of the Residual block and consequently the accuracy of the NF model.

\vspace{-0.1in}
\subsection{Continuous-time Normalizing Flow based on neural ODE}


Continuous-time NFs \cite{grathwohl2018ffjord} are implemented under the neural ODE framework \cite{chen2018neural},
where the neural network features $x(t)$ is computed by integrating an ODE as introduced in \eqref{eq:flownet},
and $v(x,t)$ is parametrized by a neural ODE network. 
We use $v_\theta$ to denote the parametrization.
 A notable advantage of the continuous-time NF is that the neural ODE framework allows to compute the forward/inverse flow mapping as well as the likelihood by numerical integration. 

\begin{itemize}
\item Forward/inverse flow mapping:

In the continuous-time flow, invertibility is presumed since a well-posed (neural) ODE can be integrated in two directions of time alike, as illustrated in Fig. \ref{fig:reverse_time_ODE}. 
Specifically, let $x(t)$ be the ODE solution trajectory satisfying $\dot x(t) = v_\theta( x(t), t)$ on $[0,T]$, the flow mapping  can be written as
\begin{equation}
  F_\theta(x) = x +  \int_0^T v_\theta(x(t),t) dt, \quad x(0) = x,
  \label{eq:map_def}
  \end{equation}
The inverse mapping $F_\theta^{-1}$ can be computed by integrating reverse in time. 
    In actual implementation, these integrals are calculated on a discrete time grid on $[0,T]$ using standard numerical integration schemes \cite{chen2018neural}.

\item Likelihood by instantaneous change-of-variable:

For flow trajectory $x(t)$ satisfying \eqref{eq:flownet}, suppose $x(0)$ has a distribution $x(0) \sim p = p_0$ and this induces the marginal density $p_t$ of $x(t)$.
The so-called instantaneous change-of-variable formula  \cite{grathwohl2018ffjord}  gives the relation 
 \begin{equation}
 \log p_t( x(t)) - \log p_s( x(s))  =  - \int_{s}^{t} \nabla \cdot v_\theta( x(\tau), \tau) d\tau,
 \label{eq:Instantaneous change-of-variable}
 \end{equation}
 which involves the time-integration of the trace of the Jacobian of $ v_\theta$. In practice, the divergence term averaged over samples is often estimated by  the Hutchinson approximation.
 While these computations may still encounter challenges in high dimensions, the ability to evaluate the (log) likelihood is fundamentally useful; in particular, it allows for evaluating the maximum likelihood training objective on finite samples. This property is also adopted in the deterministic reverse process in diffusion models \cite{song2021score}, called the ``probability flow ODE'', so the likelihood can be evaluated once a forward diffusion model has been trained.
 
\end{itemize}

Because the invertibility and likelihood computation is guaranteed by the continuous-time formulation, there is no need to design special architecture in the neural network parametrization of $v_\theta$, which makes the continuous-time NF ``free-form'' \cite{grathwohl2018ffjord}.
This allows to leverage the full expressive power of deep neural network layers 
as well as problem-specific layer types in each Residual block $f_n$ depending on the application at hand. 
For example, in additional to the basic feed-forward neural networks,
one can use convolutional neural networks (CNN) if the data are images, and graph neural networks (GNN) to generate data on graphs.

\vspace{-0.1in}
\subsection{Discrete-time flow as iterative steps}\label{subsec:discrete-flow-as-steps}

We first would like to point out that the distinction between discrete-time versus continuous-time flow models is not strict since continuous-time flow needs to be computed on a discrete time grid in practice
-- recalling that the ResNet block \eqref{eq:resnet} itself can be viewed as a Forward Euler scheme to integrate an ODE. 
In particular, one can utilize the benefit of continuous-time NF (neural ODE) inside the discrete-time NF framework by setting the $n$-th block $f_n$ to be a neural ODE on a subinterval of time. 

Specifically, let the time horizon $[0,T]$ be discretized into $N$ subintervals $[t_{n-1}, t_n]$
and $x(t)$ solves the ODE with respect to the velocity field $v_\theta(x,t)$.
 The $n$-th block mapping (associated with the subinterval $[ t_{n-1}, t_n]$)
  is defined as 
  \begin{equation}
  x_n  = x_{n-1}+\int_{t_{n-1}}^{t_n} v_\theta(x(t),t) dt, \quad x(t_{n-1}) = x_{n-1}.\label{transport_map_n} 
  \end{equation}
  This allows the computation of the likelihood via integrating $\nabla v_\theta$ by \eqref{eq:Instantaneous change-of-variable} (and concatenating the $N$ subintervals), 
 and the inverse mapping $x_n \mapsto x_{n-1}$ of each block again can be computed by integrating the ODE reverse in time. 
 
In short, by adopting a continuous-time NF sub-network inside each Residual block, one can design a discrete-time flow model that is free-form, 
automatically invertible (by using small enough time step to ensure sufficiently accurate numerical integration of the ODE such that the ODE trajectories are distinct), and enjoys the same computational and expressive advantage as continuous-time NF. 
One subtlety, however, lies in the parametrization of $v_\theta$: in the standard continuous-time NF, $v_\theta(x,t)$ is ``one-piece'' from time $0$ to $T$, 
while when putting on a discrete time grid with $N$ time stamps, $v_\theta( x,t)$ on the subinterval $[t_{n-1}, t_{n}]$ provides the parametrization of $f_n$ in  \eqref{eq:resnet} and can be parametrized independently from the other blocks (as is usually done in ResNet). 

If using independent parametrization of $v_\theta$ on $[t_n, t_{n-1}]$, the $n$-th block can potentially be trained independently and {\it progressively}
-- meaning that only one block is trained at a time and 
the \(n\)-th block is trained only after the previous \( (n-1) \) blocks are fully trained and fixed (see Fig. \ref{Fig:progressive}) 
-- but the training objective needs to be modified from the end-to-end likelihood \eqref{eq:log-ptheta-NF}. 
Such training of flow models in a progressive manner has been implemented under various context in literature,
particularly in \cite{alvarez2022optimizing,mokrov2021large,fan2022variational,xu2022jko,vidal2023taming} motivated by the Jordan-Kinderleherer-Otto (JKO)  scheme, which we will detail more in Section \ref{sec:flow-iterative}. 
The progressive training intuitively enables incremental evolution of probability distributions over time.
Experimentally, it has been shown to improve the efficiency of flow-based generative models (by reducing computational and memory load in training each block) while maintaining high-quality sample generation. From a theoretical point of view, the discrete-time Residual blocks in such flow models can naturally be interpreted as ``steps'' in certain iterative Gradient Descent scheme that minimizes a variational objective over the space of probability densities.

\vspace{-0.1in}
\subsection{Simulation-free training: Flow Matching}
\label{subsec:review-FM}

While continuous-time NFs enjoy certain advantages thanks to the neural ODE formulation,
a computational bottleneck for high dimensional data is the computation of $\nabla \cdot v_\theta$ in \eqref{eq:Instantaneous change-of-variable}.
The backpropagation training still needs to track the gradient field along the numerical solution trajectories of the neural ODE, which makes the approach ``simulation-dependent'' and computationally costly.
In contrast, the recent trend in deep generative models focuses on ``simulation-free'' approaches, where the training objective is typically an $L^2$ loss (mean squared error) that ``matches'' the neural network velocity field $v_\theta(x,t)$ to certain target ones. 
Such simulation-free training has been achieved by Diffusion Models \cite{ho2020denoising,song2021scorebased} as well as Flow Matching (FM) models \cite{lipman2023flow,albergo2023building,liu2022rectified}.
FM ensures that the learned velocity field satisfies CE, is computationally efficient, and potentially enables efficient sample generation with fewer steps.
FM models have demonstrated state-of-the-art performance across various applications, such as text-to-image generation
and audio generation. 
\rev{A recent review of FM formulated via probability paths can be found in \cite{lipman2024flow}.}

Here we provide a brief review of the latter, primarily following the formulation in \cite{albergo2023building}.
The FM model still adopts continuous-time neural ODE and the time interval is  $[0,1]$. 
FM utilizes a pre-specified ``interpolation function'' $I_t$, parametrized by $t \in [0,1]$, which smoothly connects samples for two endpoints $x_0$ and $x_1$ defined as
\begin{equation}\label{eq:interpolation}
    \phi(t) : =I_t(x_0, x_1), \quad t \in [0,1], 
\end{equation}
where $x_0 \sim p$, $x_1 \sim q$. Common interpolation function is a straight line from $x_0$ to $x_1$.
The model is trained to match the velocity field \( v_\theta(x,t) \), denoted as $\hat v$ here, 
to the true probability flow induced by \( I_t \) via minimizing the (population) loss
\begin{equation}\label{eq:fm_loss}
L( \hat v) := \int_0^1 \mathbb{E}_{ x_0, x_1} \left\| 
  \hat v ( \phi(t),t)-\frac{d}{dt} \phi(t)\right\|^2 dt.
\end{equation}
Here we suppress the parameter $\theta$ in notation as we assume sufficient expressiveness of the flow network 
and for simplicity consider the unconstrained minimization of the field $\hat v(x,t)$.

 

While not immediate from its appearance, the minimization of \eqref{eq:fm_loss} gives a desired velocity field that leads us to the correct target. Formally, we call a velocity field $v(x,t)$ on $\R^d \times [0,1]$ ``valid'' if it  provides a desired transport  from $p$ to $q$, i.e., the continuity equation (CE) $\partial_t \rho + \nabla \cdot (\rho v) = 0$ starting from $\rho(\cdot,0) = p$ satisfies $\rho(\cdot,1) = q$.
%
%
It can be shown that there is a valid $v$ (depending on the choice of $I_t$) such that, up to a constant, $L( \hat v )$ is equivalent to the $L^2$ loss 
$ \int_0^1 \int_{\R^d} \| \hat v(x,t) - v(x,t) \|^2 \rho_t(x) dx dt$ where $\rho_t$ solves the CE induced by $v$. 
This has been derived in  \cite{albergo2023building}; see also \cite{xu2024local}.


\begin{lemma}[Consistency of FM loss]\label{lemma:FM_loss}
Given $p$, $q$ and the interpolation function $I_t$, 
there exists a valid velocity field $v$ such that,
with $\rho_t(x)=\rho(x,t)$ being the solution of the induced CE by $v$, the loss $L( \hat v )$ can be written as
\begin{equation}\label{eq:loss-FM-2}
L(\hat v) = c + \int_0^1 \int_{\R^d} \| \hat v(x,t) - v(x,t) \|^2 \rho_t(x) dx dt, 
\end{equation}    
where $c$ is a constant independent from $\hat v$. 
\end{lemma}
A direct result of the lemma is that the (unconstrained) minimizer of the MF loss \eqref{eq:fm_loss} is $\hat v = v$ and it is a valid velocity field.
In practice, the training of $\min L(\hat v)$ from finite samples is via the empirical version of  \eqref{eq:fm_loss}.

\vspace{-0.1in}
 \section{Flow as iterative algorithm in probability space}\label{sec:flow-iterative}


In this section, we elaborate on flow models that implement iterative steps to minimize a variational objective in the Wasserstein space. 
We will detail on the flow motivated by the JKO scheme and the theoretical analysis of the generation accuracy.

\subsection{Iterative flow using JKO scheme}\label{subsec:jko-flow-detail}

\begin{figure}[t]
\centering
\includegraphics[height=0.35\linewidth]{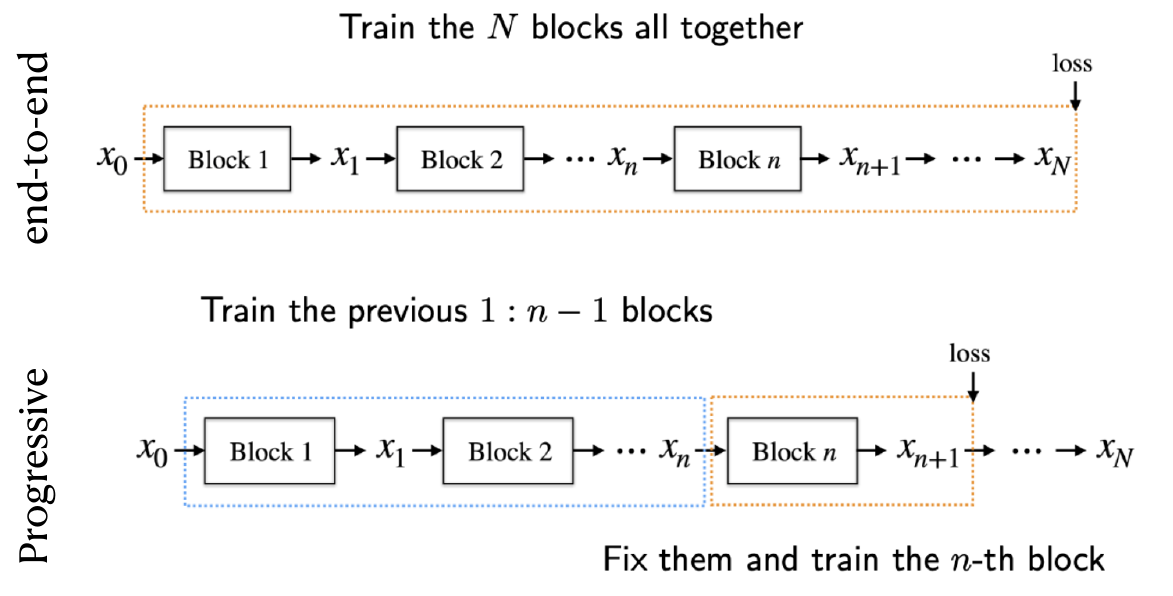}
 \caption{End-to-end versus progressive training of a flow model consisting of $N$ Residual blocks.}
 \label{Fig:progressive}
 \vspace{-0.2in}
 \end{figure}

We consider discrete-time flow models where each block can potentially take form of a continuous-time NF (neural ODE), following the set up in Section \ref{subsec:discrete-flow-as-steps}.
As has been shown in Section \ref{sec:algo-basics}, 
in the (end-to-end) training of a flow model
all the $N$ blocks (or the entire flow on \([0,T]\)) are optimized simultaneously using a single objective, usually the max-likelihood.
In contrast, the progressive training will train each block sequentially and independently, attaching a loss specific to the block during training, see Fig. \ref{Fig:progressive}. 
We call such flow implementing the iterative steps the {\it iterative flow},
and the key is to design a step-wise loss to train each block. 

For iterative flow models motivated by the JKO scheme (here our presentation illustrates the framework from \cite{xu2022jko,cheng2024convergence}), the step-wise loss is the Kullback-Leibler (KL) divergence to the known target distribution $q \propto e^{-V}$, namely
\begin{equation}\label{eq:def-KL-G}
 {\rm KL} (\rho || q)
    =  \int \rho(x) \log \rho(x) dx + \int V(x) \rho(x) dx + \text{const}.
\end{equation}
Specifically, the classical JKO scheme \cite{jordan1998variational} computes a sequence of distributions $\rho_n$, $n=0,1,...$ by 
\begin{align}\label{eq:JKO-obj-1}
    \rho_{n+1} = \text{arg} \min_{\rho\in \calP_2 }  
     {\rm KL} (\rho || q) + \frac{1}{2 \gamma} \W^2(\rho_n, \rho),
\end{align}
starting from $\rho_0\in \calP_2$, where $\gamma > 0$ controls the step size. 
In the context of normalizing flow, 
the sequence starts from $\rho_0 = p$ the data density and the density $\rho_n$ evolves to approach $q$ as $n$ increases. 

Strictly speaking, the minimization in \eqref{eq:JKO-obj-1} is over the Wasserstein-2 space of the density $\rho$, 
which, in the $n$-th JKO flow block will apply to the pushforwarded density by the mapping in the $n$-th block.
In other words, define the forward mapping \eqref{transport_map_n} in the $n$-th block as $F_n$, 
i.e. $x_n = F_n( x_{n-1})$, and in view of \eqref{eq:resnet}, $F_n(x) = x + f_n(x)$; 
The parametrization of $F_n$ is via $v_\theta(x,t)$ on $t \in [t_{n-1}, t_n]$, so we write $F_{n,\theta}$ to emphasize the parametrization. 
We denote by $p_n$ the marginal distribution of $x_n$, where $p_0 = p$ ($x_0$ follows the data distribution), and then we have 
\begin{equation}
p_n = (F_{n, \theta})_\# p_{n-1}.
\end{equation}
Following \eqref{eq:JKO-obj-1}, the training of the $n$-th JKO flow block is by
\begin{align}\label{eq:JKO-obj-2}
     \min_\theta ~ {\rm KL}(  ({F_{n,\theta}})_\# {p_{n-1}} \| q ) 
     + \frac{1}{2 \gamma} \W^2( p_{n-1}, ({F_{n,\theta}})_\# {p_{n-1}} ),
\end{align}
which is equivalent to the following objective \cite{xu2022jko}
 \begin{equation}
 \min_\theta ~ {\rm KL}(  ({F_{n,\theta}})_\# {p_{n-1}} \| q ) + \frac{1}{2\gamma} {\mathbb E}_{x\sim p_{n-1}}  \| x-F_{n,\theta} (x) \|^2. 
 \label{eq:iter_n}
 \end{equation}
 When $q = \calN(0,I)$, we have $V(x) = \|x\|^2/2$.
 Then, by the instantaneous change-of-variable formula \eqref{eq:Instantaneous change-of-variable}, the KL divergence term in \eqref{eq:iter_n} expands to
 \begin{equation}\label{eq:KL-term-expand}
 {\rm KL}(  ({F_{n,\theta}})_\# {p_{n-1}} \| q ) =
 \mathbb E_{x(t_{n-1}) \sim p_{n-1}}  
  	\left(  \frac{x( t_{n})^2}{2}  -  \int_{ t_{n-1} }^{t_{n} }\nabla\cdot v_\theta ( x (\tau),\tau)d\tau  \right) + \text{const.}
 \end{equation}

 \begin{figure}[b]
\vspace{-0.2in}
\centering
\includegraphics[height=0.35\linewidth]{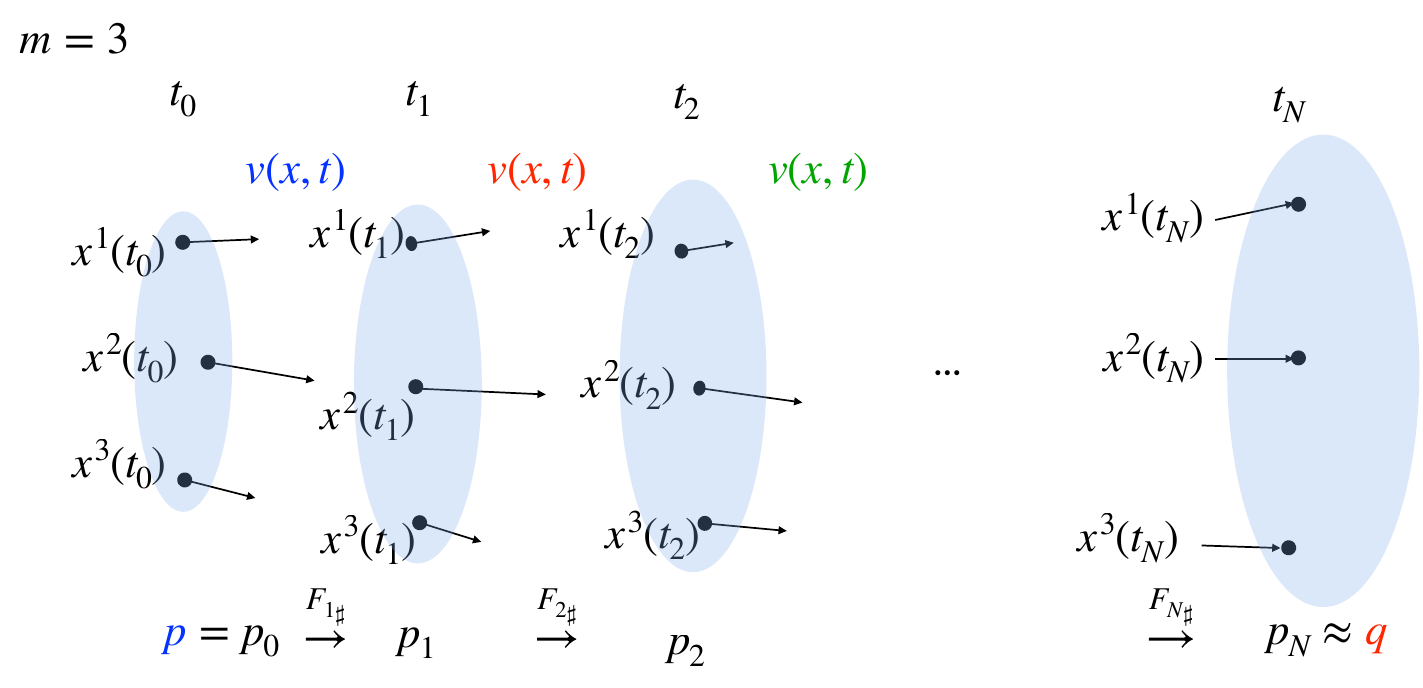}
 \caption{Example of JKO flow scheme based on particle implementation, see more in Section \ref{subsec:jko-flow-detail}. 
 }
 \label{Fig:JKO_eg}
 \vspace{-0.1in}
 \end{figure}

\paragraph{Example: JKO flow as pushforwarding particles}

The minimization of (the empirical version of) \eqref{eq:iter_n}\eqref{eq:KL-term-expand} in each step of the JKO flow is ready to be computed on particles $\{ x^i \}_{i=1}^m$, namely the finite data samples. To illustrate such an iteration, consider an example with \( m = 3 \) particles, as shown in Fig. \ref{Fig:JKO_eg}. Initially, at \( t_0 = 0 \), the particle positions correspond to training samples $x^i( 0) = x^i$.  In the first iteration, we train the first residual block, and the velocity field \( v_\theta(x, t) \) over the interval \( t \in [t_0, t_1] \) is modeled by a neural network with parameters \( \theta \). 
The empirical version of  \eqref{eq:iter_n}\eqref{eq:KL-term-expand} gives the training objective of the first block as 
 \begin{equation}\label{eq:objective-jko-example-n=1}
 \min_\theta~ \frac{1}{m} \sum_{i=1}^m \left( \frac{x^i( t_1)^2}{2}  -  \int_{ t_0 }^{t_{1} }\nabla\cdot v_\theta ( x^i (\tau; \theta),\tau) d\tau  \right) 
 	+  \frac {1}{2\gamma m}\sum_{i=1}^m \| x^i(t_{1}) - x^i(t_0) \|^2,
 \end{equation}
 where \( \gamma > 0 \) controls the step size.
 After the first block is trained, the particle positions are updated using the learned transport map \eqref{transport_map_n} on $[t_0, t_1]$, namely,
 \[
 x^i(t_1) = x^i(t_0) + \int_{t_0}^{t_1} v_\theta( x^i(t),t)dt, \quad \dot x^i (t) = v_\theta( x^i(t), t ), \quad x^i(t_0) = x^i.
 \]
 In the next iteration, we train the velocity field \( v_\theta(x, t) \) over the time interval \( [t_1, t_{2}] \), and the initial positions of the particles are $x^i(t_1)$ which have been computed from the previous iteration. This procedure continues for $n=1, 2, \cdots, N$ for $N$ steps (Residual blocks).

\vspace{-0.1in}
\subsection{Interpretation of Wasserstein regularization term}

Before imposing the flow network parametrization, the original JKO scheme \eqref{eq:JKO-obj-1} can be interpreted as the  $\W$-proximal Gradient Descent (GD) of the KL objective  \cite{salim2020wasserstein,cheng2024convergence} with step size controlled by $\gamma$. 
This naturally provides a variational interpretation of the iterative flow model as implementing a discrete-time GD on the $\W$ space.
The connection to Wasserstein GD allows us to prove the generation guarantee of such flow models by analyzing the convergence of proximal GD in Wasserstein space, to be detailed in Section \ref{subsec:theory-convergence}.

Meanwhile, the per-step training objective \eqref{eq:JKO-obj-2} can be viewed as the addition of the variational objective
(closeness of the pushforwareded density to target $q$)
and the Wasserstein term (the squared $\W$ distance between the pushforwareded density and the current density).
The  Wasserstein term serves to {\it regularize} the ``amount of movement'' from the current density $p_{n-1}$ by the transport map $F_{n,\theta}$. 
Intuitively, among all transports $F_{n,\theta}$ that can successfully reduces the KL divergence from $(F_{n, \theta})_\# p_{n-1}$ to $q$, the regularization term will select the one that has the smallest movement.
By  limiting excessive movement from one iteration to the next, 
the $\W$ regularization term leads to straighter transport paths in probability space, as illustrated in Fig. \ref{Fig:regularization}.
It may also reduce the number of neural network blocks needed to reach the target distribution. 
In a particle-based implementation of the $n$-th flow block, this $\W$ term can be computed as 
$\frac 1m\sum_{i=1}^m \|{x^i(t_{n})}-{x^i(t_{n-1})}\|^2$, the average squared movement over $m$ particles, making use of the ODE trajectories $x^i(t)$.



\begin{figure}[t]
\centering
\includegraphics[height=0.3\linewidth]{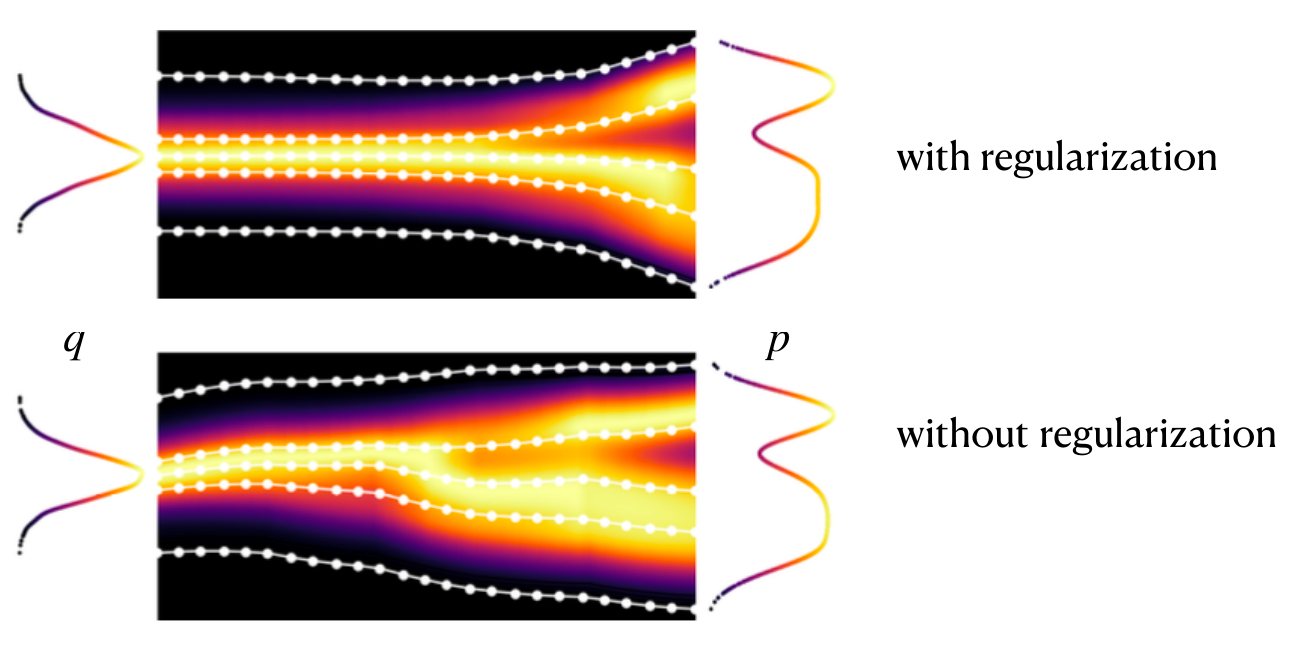}
 \caption{Trajectory of flow between \( p \) and \( q \) with and without regularization: Regularization results in ``straighter paths" for the particles, requiring fewer steps to transition from \( p \) to \( q \) and reducing the number of neural network blocks in the implementation.}
 \label{Fig:regularization}
 \vspace{-0.2in}
 \end{figure}


\vspace{-0.1in}
\subsection{Simulation-free iterative flow via local Flow Matching}

\rev{
As shown in Section \ref{subsec:jko-flow-detail},
each JKO step \eqref{eq:JKO-obj-1} aims at descending the KL-divergence using a Wasserstein proximal GD.
This can be viewed as a likelihood-based training, 
which enjoys good theoretical property at price of a computational drawback namely that the training is not simulation-free.
It would be desirable to incorporate simulation-free training, such as Flow Matching (FM), under the iterative flow framework. }
To this end, \cite{xu2024local} developed Local Flow Matching by introducing an iterative, block-wise training approach, where each step implements a simulation-free training of an FM objective. 
\rev{In each step, the source distribution $p_{n-1}$ is the current data distribution, 
and the target distribution $p_{n}^*$ is by evolving from $p_{n-1}$ along the OU process for a time step $\gamma_n > 0$. 
In other words, $x_r\sim p_{n}^*$ can be produced by 
$x_r : = e^{-\gamma_n } x_l + \sqrt{ 1-e^{-2 \gamma_n} } g $
where  $x_l \sim p_{n-1}$  and $g \sim  \calN(0,I_d)$.
The $n$-th sub-flow model, once trained, pushforards data distribution $p_{n-1}$ to $p_n \approx p_{n}^*$.
Because the OU process solves the continuous-time Wasserstein gradient flow that minimizes the KL-divergence,
we expect each step in local FM also pushes the sequence of distributions $p_n$ closer to the final density $q$, similarly as in the JKO-flow model. 
This intuition will be used in the convergence analysis below.
}

In terms of model design, the previous (global) FM model directly interpolates between noise and data distributions, which may differ significantly. 
In contrast, Local FM decomposes this transport into smaller, incremental steps, interpolating between distributions that are closer to each other, hence the name ``local.’’ 
\rev{
The Local FM model trains a sequence of invertible sub-flow models, 
which, when concatenated, transform between the data and noise distributions; a real data example is shown in Fig. \ref{Fig:LFM}. 
Thanks to the decomposition of the global trajectory, 
in each local step the two distributions to connect are closer to each other,
and as a result an FM sub-flow of a smaller model size can be used. 
Using the same overall model size, local FM shows better training efficiency 
than (global) FM by achieving lower FID with the same or fewer number of batches in SGD, see the details in  \cite{xu2024local}.}

\vspace{-0.1in}
\subsection{Theoretical analysis of generation guarantee}\label{subsec:theory-convergence}

Most theoretical results on deep generative models focus on 
score-based diffusion models (the forward process is always an SDE), 
e.g. the latest ones like \cite{
chen2024probability,li2024towards},
and (end-to-end training, global) flow models (in both forward and reverse processes) 
such as recent works: for the flow-matching model  \cite{benton2024error} and applied to probability flow ODE in score-based diffusion models;
for neural ODE models trained by likelihood maximization (the framework in \cite{grathwohl2018ffjord}) \cite{marzouk2023distribution}.

Below, we highlight a key insight in analyzing iterative flow models 
by making the connection to the convergence of Wasserstein GD.
Because the Wasserstein GD will be shown to have linear (exponential) convergence,
such analysis bounds the needed number of iterative steps $N$ (number of Residual blocks) to be $\sim \log (1/\varepsilon)$ for the flow model to achieve an $O(\varepsilon)$ amount of ``generation error'',
which is measured by the divergence between the generated distribution and the true data distribution. 
Here we follow the proof for the JKO flow model \cite{cheng2024convergence}, and similar idea has been applied to prove the convergence of Local FM \cite{xu2024local}.



To study the evolution of probability densities, we adopt the Wasserstein space as the natural setting, as it captures the geometric structure of probability distributions via transport maps. 
Recall the iterative scheme in Section \ref{subsec:jko-flow-detail} produces a sequence of distributions in the $\W$ space from data to noise and back, which we denote as the {\it forward process} and the {\it reverse process} respectively: 
\begin{equation}\label{eq:fwd-bwd-process} 
\begin{split}
\text{(forward)} \quad 
& 
p = p_0  
\xrightarrow{F_1}{p_1}  
\xrightarrow{F_2}{} 
\cdots 
\xrightarrow{F_{N}}{p_N} 
\approx q,  \\ 
\text{(reverse)} \quad 
&  p \approx
q_0  \xleftarrow{F_1^{-1}}{ q_1}
\xleftarrow{F_2^{-1}}{ }
\cdots 
\xleftarrow{F_{N}^{-1}}{ q_N}
= q,
\end{split}
\end{equation}
The density $q_0$ is the generated density by the learned flow model, and the goal is to show that $q_0$ is close to data density $p = p_0$. The proof framework consists of establishing convergence guarantees
first for the forward process and consequently for the reverse process: 

\paragraph{Forward process (data-to-noise) convergence}

\begin{itemize}
\item 
 
At each iteration of minimizing \eqref{eq:JKO-obj-2} which gives a pushforwarded $p_n$ by the learned $F_{n,\theta}$, 
 we assume that the minimization is approximately solved with the amount of error $O(\varepsilon)$ that is properly defined.
 
\item The forward convergence guarantee is by mirroring the analysis of vector space (proximal) GD for convex optimization analysis. Specifically, making use of the \( \lambda \)-convexity of \( G(\rho) := {\rm KL}(\rho \| q) \) in Wasserstein space, one can show the Evolution Variational Inequality (EVI):
   \[
        \left(1+\frac{\gamma \lambda}{2}\right)
        \mathcal{W}^2(p_{n+1}, q) 
        + 2\gamma \left( G(p_{n+1}) - G(q) \right)
        \leq 
        \mathcal{W}^2(p_n, q) + O(\varepsilon^2).
     \]

\item  The EVI is a key step to establish the exponential convergence of the Wasserstein GD and the guarantee of closeness between $p_n$ to $q$ in KL-divergence. Specifically, \( O(\varepsilon^2) \) KL-divergence is obtained after approximately \( \log(1/\varepsilon) \) JKO steps (Residual blocks).

\end{itemize}

\paragraph{Reverse process (noise-to-data) convergence}  
The convergence of the reverse process follows by the invertibility of flow map, and we utilize the following key lemma,  Bi-direction Data Processing Inequality (DPI) for $f$-divergence.
Let $p$ and $q$ be two probability distribution, such that $p$ is absolutely continuous with respect to $q$. Then for a convex function $f: [0, \infty)\rightarrow \mathbb R$ such that $f$ is finite for all $x>0$, $f(1) = 0$, and $f$ is right-continuous at 0, the $f$-divergence of $p$ from $q$ is defined as $D_f(p\|q) = \int f\left(\frac{p(x)}{q(x)}\right)q(x)dx$; for instance, KL-divergence is an $f$-divergence.
\begin{lemma}[Bi-direction DPI]\label{lemma:bi-DPI}
Let ${\rm D}_f$ be an $f$-divergence.
    If $F: \R^d \to \R^d$ is invertible 
    and for two densities $p$ and $q$  on $\R^d$,
    $F_\# p$ and $F_\# q$ also have densities, then 
    \[
    D_f(p || q)  = D_f (F_\# p ||F_\# q) . 
    \]
\end{lemma}
The proof is standard and can be found in \cite{xu2024local,raginsky2016strong}. The DPI controls information loss in both forward and reverse transformations. 
Bringing these results together provides a density learning guarantee. 
Because KL-divergence is an $f$-divergence, the closeness at the end of the forward process in terms of ${\rm KL}(p_N \| q_N =q)$
directly implies the same closeness at the end of the reverse process, i.e. ${\rm KL}(p_0=p \| q_0)$; see \eqref{eq:fwd-bwd-process}.
The \( O(\varepsilon^2) \) KL control implies $O(\varepsilon)$ bound in  Total Variation (TV). 

 \begin{figure}[t]
\centering
\includegraphics[width=0.8\textwidth]{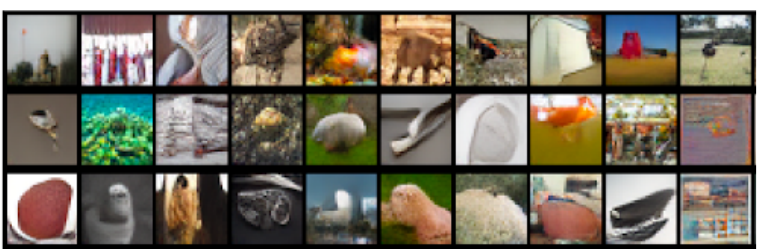}
 \caption{Example of JKO-flow generated synthetic images, 
 trained using ImageNet-32 dataset with 1,281,167 training samples. 
 Labels identified for the first 10 images (first row) by a classification algorithm are 1: `flagpole', 2: `ski', 3: `Windsor tie', 4: `Gila monster', 5: `goldfish', 6: `warthog', 7: `admiral', 8: `envelope', 9: `barn', 10: `ostrich'.
 \rev{JKO-flow model achieves an FID score of 20.1, and local FM model achieves FID 7.0 \cite{xu2024local}.}}
 \label{Fig:JKO-generation}
 \vspace{-0.1in}
 \end{figure}

\vspace{-0.1in}
\section{Applications and extensions}

In this section, we present various applications and extensions of the previously discussed flow-based generative models to problems in statistics, signal processing, and machine learning, involving a general target density and a general loss function.

\vspace{-0.1in}
\subsection{Data synthesis and evaluation metrics}

Synthetic data generation is a common application of generative models, aiming to learn complex data distributions of the training data and synthesize new data samples that follow the same distribution. In image generation, generative models effectively capture intricate patterns, textures, and structures from large datasets, enabling them to generate high-quality images that closely resemble real-world data. This capability has numerous applications in computer vision, content creation, and medical imaging, where synthetic images can enhance training datasets and enable dowstream machine learning tasks.




For data synthesis tasks, generative models are evaluated using various metrics that assess the quality, diversity, and likelihood of the generated samples. Commonly used metrics include Fréchet Inception Distance (FID), Negative Log-Likelihood (NLL), Kullback-Leibler (KL) Divergence, Log-Likelihood per Dimension, Precision and Recall for Distributions, Inception Score (IS), and Perceptual Path Length (PPL). Other metrics for comparing two sets of sample distributions can also be used to evaluate generation quality, such as the kernel Maximum Mean Discrepancy (MMD) statistic. Among these, FID and NLL are the most frequently used for assessing both perceptual quality and likelihood estimation in flow-based models. In practice, researchers often report both FID and NLL when evaluating flow-based generative models.

FID is a widely used metric that measures the distance between the feature distributions of real and generated images. It is computed as the Fr\'echet distance between two multivariate Gaussian distributions, one representing real images and the other representing generated images. Mathematically, FID is given by:  
$\|\mu_r - \mu_g\|^2 + \text{Tr}(\Sigma_r + \Sigma_g - 2(\Sigma_r \Sigma_g)^{1/2})
$
where \(\mu_r, \Sigma_r\) are the mean and covariance matrix of data in feature space, and \(\mu_g, \Sigma_g\) are the mean and covariance matrix of generated data. The first term captures differences in mean, while the second term accounts for differences in variance. Lower FID values indicate better sample quality and diversity. FID provides an empirical assessment of whether the generated samples match the real data distribution in a perceptual feature space, which is important in applications where sample quality is crucial, such as image generation.

Negative Log-Likelihood (NLL) is a direct measure of how well a generative model fits the training data distribution. 
Given a test dataset \(\{\tilde x^i\}_{i=1}^{m'}\) and a model (e.g., learned by flow model) with probability density function \(p_{\theta}(x)\), the NLL is computed as:  $
\text{NLL} = -\frac{1}{m'} \sum_{i=1}^{m'} \log p_{\theta}(\tilde x^i)
$
where \(p_{\theta}(x)\) is the model’s probability density at test sample \(\tilde x^i\). Lower NLL values indicate that the model assigns high probability to real data, meaning it has effectively captured the distribution. Conversely, higher NLL values suggest poor data fit.  Since flow-based models explicitly learn the data distribution, NLL serves as a fundamental evaluation metric, complementing perceptual metrics like FID,  but it does not always correlate with perceptual sample quality. Fig. \ref{local-flow} illustrated NLL scores for various generative models.

\begin{figure}[t]
\centering
\includegraphics[height=0.3\linewidth]{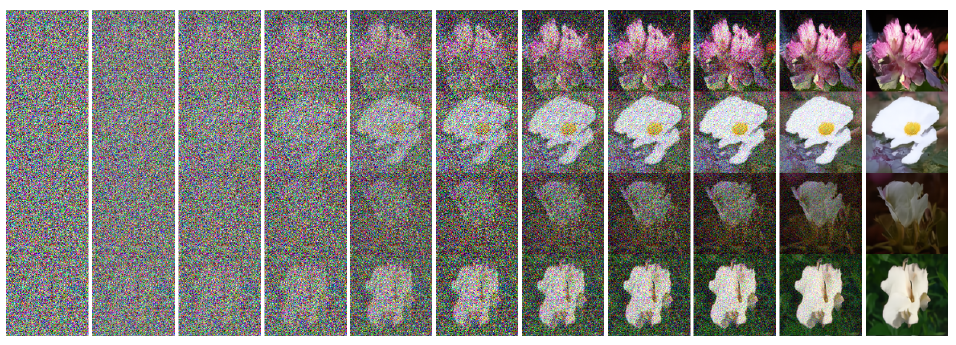}
 \caption{
 Noise-to-image trajectories of 
Oxford Flowers \citep{Nilsback2008flower} data (128 x 128)
by
 Local Flow Matching \cite{xu2024local}.}
 \label{Fig:LFM}
 \vspace{-0.1in}
 \end{figure}

\vspace{-0.1in}
\subsection{General $q$ rather than Gaussian}

In many applications, the goal is to learn a mapping from \( p \) to a general distribution \( q \) rather than restricting \( q \) to be a Gaussian. 
For instance, in image applications, generative models can interpolate between different styles, manipulate image attributes, and generate high-resolution outputs, making them powerful tools for various real-world tasks. This has significant implications for transfer learning, domain adaptation, and counterfactual analysis.


\rev{
As has been shown in Section \ref{subsec:review-FM}, there is more than one valid interpolating trajectory connecting a source density $p$ to a target density $q$.
When requiring the interpolation also minimizes the ``transport cost,'' the problem is equivalent to the optimal transport (OT) problem as revealed by the classical dynamic formulation of OT.
This task aligns with the Schr\"{o}dinger bridge (SB) problem, which seeks stochastic processes that interpolate between given marginal distributions
and can be viewed as an entropy-regularized version of the OT problem \cite{leonard2013survey}.
Many works have explored these connections to leverage OT and SB concepts in developing flow and diffusion generative models.
The formulation is also closely related to the Stochastic Control problem, which can be solved under a deep flow-model framework, e.g., in \cite{domingo2024stochastic}.
As this is a broad topic tangent to the theme of the current paper, below we briefly show one example approach based on  \cite{xu2023computing}
focusing on scaling OT computations to high-dimensional data using a (deterministic) flow model, without going into general stochastic control formulations.
A statistical application is used to estimate density ratios in high dimensions.}

\paragraph{Flow-Based Optimal Transport}


Recall that the celebrated Benamou-Brenier equation expresses OT in a dynamic setting \citep{villani2009optimal, benamou2000computational}:
 \begin{equation}\label{eq:Benamou-Brenier}
 \begin{split}
 &\mathcal T:= \inf_{\rho, v}      \int_0^1 \mathbb{E}_{x(t) \sim \rho(\cdot,t)} \| v(x(t),t) \|^2 dt  \\
 & \text{s.t.} ~~\partial_t \rho + \nabla \cdot (\rho  v) = 0, 
   ~~  \rho(\cdot,0) = p,  ~~ \rho(\cdot,1) = q,    
 \end{split}
 \end{equation}
 where \( v(x,t) \) represents the velocity field, and \( \rho(x,t) \) is the probability density at time \( t \), evolving according to the CE. Under suitable regularity conditions, the minimum transport cost \( \mathcal{T} \) in \eqref{eq:Benamou-Brenier} equals the squared Wasserstein-2 distance, and the optimal velocity field \( v(x,t) \) provides a control function for the transport process.  
\rev{
As a standard approach, see also e.g. \cite{ruthotto2020machine}, 
one can incorporate the two endpoint conditions on initial/target distributions into the cost function by }
\begin{equation}   
 \int_0^1 \mathbb{E}_{x(t)\sim\rho(t)}\|v(x(t),t)\|^2_2 dt +
 \gamma {\rm KL}(p\|{\hat p}) 
 + \gamma {\rm KL}(q\|{\hat q}), \label{eq:flow-OT}
 \end{equation}
 where \( \gamma > 0 \) is a regularization parameter enforcing terminal constraints through KL divergence terms, 
 and \( \hat p  \) and \( \hat q \) denote the transported distributions obtained via the learned velocity field using CE. 
 The symmetry and invertibility of the transport map allow us to impose constraints on both forward and reverse mappings. 
A particle-based approach can approximate the discrete-time transport cost; for instance, over a segment \( [t_0, t_1] \): 
 \[
 \int_{t_0}^{t_1} \mathbb{E}_{x(t)\sim\rho_t}\|v(x(t),t)\|^2_2 dt \approx \frac {1}{(t_1-t_0)m}\sum_{i=1}^m \|x^i(t_1)-x^i(t_0)\|_2^2.
 \]
Finally, the KL divergence terms can be estimated from particles at the two endpoints by estimating the log-likelihood function $\log q(x)/p(x)$, which can be estimated using a technique training a classification network, based on Lemma \ref{thm-dre} below.

 \begin{figure}[t]
\centering
\includegraphics[width=\textwidth]{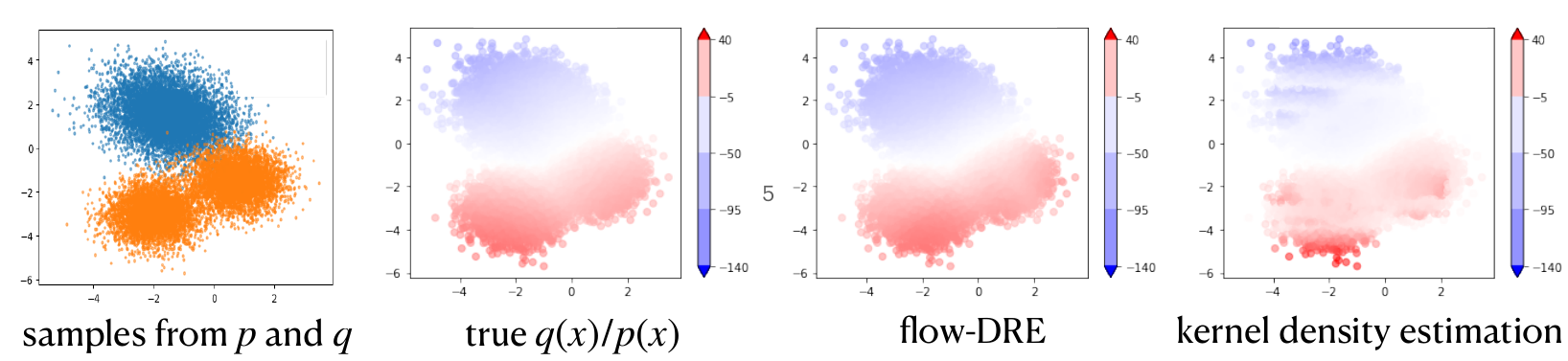}
 \vspace{-20pt}
 \caption{Density ratio between two Gaussian mixtures with very different support:  
\( p = \frac{1}{3} \mathcal{N}([-2, 2]^T, 0.75 I_2) + \frac{1}{3} \mathcal{N}([-1.5, 1.5]^T, 0.25 I_2) + \frac{1}{3} \mathcal{N}([-1, 1]^T, 0.75 I_2) \)  
and  
\( q = \frac{1}{2} \mathcal{N}([0.75, -1.5]^T, 0.5 I_2) + \frac{1}{2} \mathcal{N}([-2, -3]^T, 0.5 I_2) \),  
obtained via flow-based density ratio estimation, which provides a more accurate approximation of the true density ratio compared to standard kernel density estimation (KDE). A similar example was used in \cite{xu2023computing}; here we further compare with KDE.
} \label{Fig:DRE}
 \vspace{-0.1in}
 \end{figure}

\paragraph{Estimating Density Ratios via learned velocity fields}

The velocity field learned from \eqref{eq:flow-OT} can also be used for density ratio estimation (DRE), a fundamental problem in statistics and machine learning with applications in hypothesis testing, change-point detection, anomaly detection, and mutual information estimation. 
A major challenge in DRE arises when the supports of \( p \) and \( q \) differ significantly. To mitigate this, one technique is the telescopic density ratio estimation \cite{rhodes2020telescoping, choi2022density}, which introduces intermediate distributions that bridge between \( p \) and \( q \): Given a sequence of intermediate densities \( p_n \) for \( k=0, \dots, N \) with \( p_0 = p \) and \( p_N = q \), consecutive pairs \( (p_n, p_{n+1}) \) are chosen such that their supports are close to facilitate accurate density ratio estimation. Then the log-density ratio can then be computed via a telescopic series as:
\begin{equation}\label{eq:telescopic-dre}
 \log \frac{q(x)}{p(x)} = \log p_N(x) - \log p_0(x) = \sum_{n=0}^{N-1} \left(\log p_{{n+1}}(x) - \log p_{n}(x)\right).
 \end{equation}
 This multi-step approach improves estimation accuracy compared to direct one-step DRE.  

Using flow-based neural networks, we can learn the velocity field that transports particles between intermediate distributions and estimate density ratios at each step—for instance, by leveraging classification-based networks to distinguish between samples at consecutive time steps. The classification-based network can leads to an estimate for the log-density ratio
(using their samples) due to the following lemma,
which is straightforward to verify:
\begin{lemma}
[Training of logistic loss \rev{gives log-density ratio} under population loss and perfect training] \label{thm-dre} For logistic loss, given two distributions $f_0$ and $f_1$, let
\[\ell[\varphi] = \int \log(1+e^{\varphi(x)})f_0(x)dx+\int\log(1+e^{-\varphi(x)})f_1(x)dx,\]
Then the functional global minimizer is given by $\varphi^\star =\log(f_1/f_0)$.
\end{lemma}
An illustrative example is shown in Fig. \ref{Fig:DRE}. We remark that although this telescopic density ratio learning scheme, in principle, works with an arbitrary velocity field, an ``optimal" velocity field (e.g., one that minimizes the ``energy'') tends to be more efficient in implementation and achieves better numerical accuracy.


%
%

\begin{figure}[t]
\centering
\includegraphics[width=0.7\textwidth]{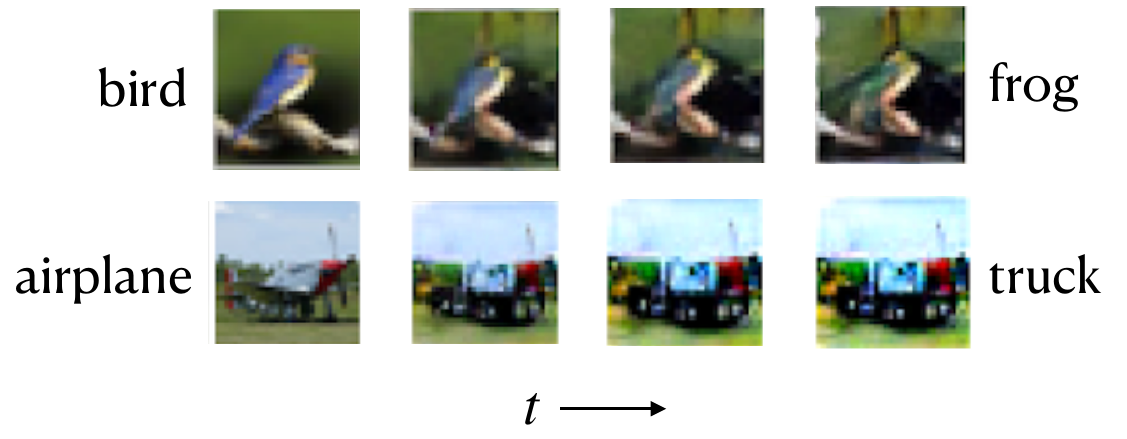}
 \vspace{-10pt}
 \caption{Illustrative examples of images generated by the adversarial sampler by solving \eqref{eq:flow_DRO}; note that after the ``worsening'' transformation, the image becomes misclassified.}
 \label{Fig:DRO}
 \vspace{-0.1in}
 \end{figure}

\vspace{-0.1in}
\subsection{Sampling from worst-case distribution}

The proposed framework can be used beyond finding distributions that match the original data distribution, but also find the worst-case distributions in distributional robust optimization (DRO) problem with application to robust learning.
\rev{
A survey of Wasserstein DRO can be found in  \cite{kuhn2019wasserstein}.
}
In this setting, the target distribution \( q \) is not pre-specified but is instead induced by a problem-specific risk function \( R(\cdot): \mathbb{R}^d \to \mathbb{R} \). 
\rev{In the so-called flow-based DRO model \cite{xu2024flow}, the problem is the solve a transport map $F: \R^d \to \R^d$ via}
\begin{equation}
 \min_{F} \mathbb{E}_{X \sim p} [R(F(X)) + \frac{1}{2\gamma} \|X - F(X)\|_2^2]
 \label{eq:flow_DRO}
\end{equation}
for a given regularization parameter $\gamma >0$.
Here, the reference measure \( p \) can be represented by a pre-trained generative model. Thus, this framework allows to adapt the pre-trained generative model from synthesize samples that follows the same distribution as the training data, to produce the worst-case samples that maximize the risk function \( R \) -- representing the unseen scenarios. 
Using a flow-based model to represent transport map \( F \), the above optimization problem can be solved by the aforementioned framework. The learned transformation \( F^* \) defines the worst-case sampler, inducing the adversarial distribution \( q=F^*_\# p \), which transforms generated samples to regions that lead to higher risk.
As an example, Fig. \ref{Fig:DRO} illustrates an adversarial sampler for classification algorithms, where the risk function \( R \) is chosen as the classification accuracy (cross-entropy loss). In this case, the generated images are mapped to cause misclassification, demonstrating how the framework can be used to generate adversarial samples from worst-case distributions in practical applications.

\section{Conclusion}

Flow-based generative models represent a compelling framework for high-dimensional probability modeling, offering advantages such as exact likelihood computation, invertible transformations, and efficient sample generation. By leveraging ODEs and optimal transport techniques, flow-based models provide a general framework for density estimation and data synthesis, making them particularly well-suited for signal processing applications. In this tutorial, we presented a mathematical and algorithmic perspective on flow-based generative models, introducing key concepts such as continuous normalizing flows (CNFs), flow-matching (FM), and iterative training via the Jordan-Kinderlehrer-Otto (JKO) scheme. Through the lens of Wasserstein gradient flows, we demonstrated how these models naturally evolve within probability space, offering both theoretical guarantees and practical scalability. Additionally, we explored extensions of flow-based models, including generalization beyond Gaussian target distributions and to include general loss for worst-case sampling via distributionally robust optimization (DRO).

There are many areas we did not cover in this tutorial. Notably, flow-based generative models can be extended for conditional generation, where the model generates data given some condition, such as class labels, textual descriptions, or auxiliary variables. This has applications in network-based conditional generation \cite{xu2022invertible,ben2024dflow}.  
\rev{Generative flow and diffusion models, including Consistency Model \cite{song2023consistency}, can also be utilized for sampling posterior \cite{chung2023diffusion,song2023loss,purohit2025consistency}.
}
Additionally, counterfactual sampling is another promising direction, allowing for the generation of synthetic samples that were not observed in real-world data but could have occurred under different causal assumptions.

By providing both theoretical insights and practical guidelines, this tutorial aims to equip researchers and practitioners with the necessary tools to develop and apply flow-based generative models in diverse domains. As these models continue to evolve, their integration with advanced optimization and machine learning techniques will further expand their impact on modern signal processing, statistical inference, and generative AI applications.



%
\vspace{-0.1in}
\section*{Acknowledgment}

This work is partially supported by NSF DMS-2134037, and the Coca-Cola Foundation.
XC is also partially supported by 
NSF DMS-2237842
and Simons Foundation (grant ID: MPS-MODL-00814643).

%

\bibliographystyle{plain}
\bibliography{references,flow}

\appendix

\section{Proofs}

The same result of  Lemma \ref{lemma:FM_loss}
was proved in \cite[Proposition 1]{albergo2023building} where it was assumed that $x_0$ and $x_1$ are independent. 
Here, we show that allowing dependence between $x_0$ and $x_1$ does not harm the conclusion.

 \begin{proof}[Proof of Lemma \ref{lemma:FM_loss}]
     We denote the {\it joint} density of $x_0, x_1$ as $\rho_{0,1}(x_0, x_1)$, and it satisfies that the two marginals are $p$ and $q$ respectively. 
     We denote $v(\cdot, t)$ as $v_t(\cdot)$ and $\rho(\cdot, t)$ as $\rho_t(\cdot)$.
 We will explicitly construct $\rho_t$ and $v_t$, 
 and then show that
 (i) $\rho_t$ indeed solves the CE induced by $v_t$, 
 and (ii) $v_t$ is valid.

 Let  $\rho_t(x)$ be the concentration of the interpolant points $I_t(x_0, x_1)$ over all possible realizations of the two endpoints, 
 that is, using a formal derivation with the Dirac delta measure $\delta$ (the mathematical derivation using Fourier representation and under technical conditions of $p$ and $q$ can be found in \cite[Lemma B.1]{albergo2023building}), we define
 \[
 \rho_t(x) : = \int_{\R^d} \int_{\R^d}  \delta( x - I_t(x_0, x_1)) \rho_{0,1}(x_0, x_1) dx_0 dx_1.
 \]
 Because $I_0(x_0,x_1) = x_0$ and $I_1(x_0,x_1) = x_1$, we know that 
 \[
 \rho_0(x) = \int \rho_{0,1}(x, x_1) dx_1 =  p(x), 
 \quad 
 \rho_1(x) = \int \rho_{0,1}(x_0, x) dx_0 =  q(x).
 \]
 By definition,
 \begin{align}
 \partial_t \rho_t (x) 
 = 
 - \int_{\R^d} \int_{\R^d} \partial_t I_t(x_0, x_1) \cdot \nabla \delta( x - I_t(x_0, x_1)) \rho_{0,1}(x_0, x_1) dx_0 dx_1 
  = -\nabla \cdot j_t(x), \label{eq:CE-jt-proof1}
 \end{align}
 where
 \[
 j_t(x): =  \int_{\R^d} \int_{\R^d} \partial_t I_t(x_0, x_1) \delta( x - I_t(x_0, x_1)) \rho_{0,1}(x_0, x_1) dx_0 dx_1.
 \]
 We now define $v_t$ to be such that 
 \[
 v_t(x) \rho_t(x) = j_t(x),
 \]
 this can be done by setting $v_t(x) = j_t(x)/\rho_t(x)$ if $\rho_t(x) > 0$ and zero otherwise. Then, \eqref{eq:CE-jt-proof1} directly gives that $\partial_t \rho_t = - \nabla \cdot ( \rho_t v_t  )$ which is the CE. This means that $v_t$ is a valid velocity field.

 To prove the lemma, it remains to show that the loss \eqref{eq:fm_loss} can be equivalently written as \eqref{eq:loss-FM-2}. 
 To see this, note that \eqref{eq:fm_loss}  can be written as 
 \begin{equation}\label{eq:loss-FM-1}
 L(\hat v) =  \int_0^1 l( \hat v , t ) dt,
 \quad
l( \hat v , t ) := \E_{x_0, x_1} \| \hat v_t ( I_t (x_0, x_1)) - \partial_t I_t( x_0, x_1) \|^2.
\end{equation}
For a fixed $t$, 
 \begin{align*}
 l( \hat v , t ) 
 & =  \int_{\R^d} \int_{\R^d} \| \hat v_t ( I_t (x_0, x_1)) - \partial_t I_t( x_0, x_1) \|^2 \rho_{0,1}(x_0, x_1) dx_0 dx_1  \\ 
 & = \int_{\R^d} \int_{\R^d} \int_{\R^d} \| \hat v_t (x) - \partial_t I_t( x_0, x_1) \|^2 
 \delta( x - I_t(x_0, x_1))  \rho_{0,1}(x_0, x_1) dx_0 dx_1 dx \\
 & = 
 c_1(t) + \int_{\R^d} \int_{\R^d} \int_{\R^d} 
 ( \| \hat v_t (x)\|^2 - 2 \hat v_t (x) \cdot \partial_t I_t( x_0, x_1)  )
 \delta( x - I_t(x_0, x_1))  \rho_{0,1}(x_0, x_1) dx_0 dx_1 dx,
 \end{align*}
 where
 \[
 c_1(t):= \int_{\R^d} \int_{\R^d} 
  \|  \partial_t I_t( x_0, x_1)  \|^2  
  \rho_{0,1}(x_0, x_1) dx_0 dx_1,
 \]
 and $c_1(t)$ is independent from $\hat v$.
 We continue the derivation as
 \begin{align*}
  l( \hat v , t )  - c_1(t)
 & = 
 \int_{\R^d}
 \| \hat v_t (x)\|^2
  \int_{\R^d} \int_{\R^d} 
 \delta( x - I_t(x_0, x_1))  \rho_{0,1}(x_0, x_1) dx_0 dx_1 dx \\
 & ~~~~~~
 - 2 \int_{\R^d} 
  \hat v_t (x) \cdot 
  \int_{\R^d} \int_{\R^d}  \partial_t I_t( x_0, x_1)  
 \delta( x - I_t(x_0, x_1))  \rho_{0,1}(x_0, x_1) dx_0 dx_1 dx \\
 & = \int_{\R^d}
         \| \hat v_t (x)\|^2 \rho_t(x) dx
         - 2 \int_{\R^d} 
         \hat v_t (x) \cdot j_t(x) dx \\
 & =  \int_{\R^d}
         ( \| \hat v_t (x)\|^2 
         - 2 \hat v_t (x) \cdot v_t(x) ) \rho_t(x) dx \\
& =   \int_{\R^d}
         \| \hat v_t (x) -  v_t (x)\|^2 \rho_t(x) dx 
         - \int_{\R^d} \| v_t (x)\|^2   \rho_t(x) dx,         
  \end{align*}
and then, by defining \[
c_2(t): = \int_{\R^d} \| v_t (x)\|^2   \rho_t(x) dx,
\] 
which is again  independent from $\hat v$, we have
  \[
    l( \hat v , t )  = \int_{\R^d}  \| \hat v_t (x) -  v_t (x)\|^2 \rho_t(x) dx  +  c_1(t)  - c_2(t).
  \]
  Putting back to \eqref{eq:loss-FM-1} we have proved the lemma, where the constant 
  $
  c =  \int_0^1 (c_1(t) - c_2(t)) dt$.
 \end{proof}

\end{document}